\documentclass[review]{elsarticle}

\usepackage{lineno,hyperref}
\usepackage{xspace}
\newcommand*{\eg}{e.g. \@\xspace}
\modulolinenumbers[5]

\usepackage{graphicx}
\usepackage{comment}
\usepackage{amsmath,amssymb} 
\usepackage{ulem}
\usepackage{enumitem}
\usepackage{amsthm}
\usepackage{subfigure}
\usepackage{algorithm,algcompatible,amsmath}
\usepackage[noend]{algpseudocode}
\algnewcommand\INPUT{\item[\textbf{Input:}]}%
\algnewcommand\OUTPUT{\item[\textbf{Output:}]}%
\usepackage{amsmath}               
  {
      \newtheorem{assumption}{Assumption}
  }

\newtheorem{definition}{Definition}
\newtheorem{theorem}{Theorem}

\pdfoutput=1

\begin{document}

\begin{frontmatter}

\title{Semantic Clustering based Deduction Learning for Image Recognition and Classification}

\author{Wenchi Ma\fnref{myfootnote1}}
\author{Xuemin Tu\fnref{myfootnote2}}
\author{Bo Luo\fnref{myfootnote1}}
\author{Guanghui Wang\fnref{myfootnote3}}
\fntext[myfootnote1]{W. Ma and B. Luo are with the Department of Electrical Engineering and Computer Science, The University of Kansas, Lawrence, KS, 66045 USA e-mail: wenchima@ku.edu; bluo@ku.edu.}
\fntext[myfootnote2]{X. Tu is with the Department of Mathematics, The University of Kansas, Lawrence, KS, 66045 USA e-mail: xuemin@ku.edu.}
\fntext[myfootnote3]{G. Wang is with the Department of
Computer Science, Ryerson University, Toronto, ON, M5B 2K3 Canada. e-mail: wangcs@ryerson.ca}



\begin{abstract}
The paper proposes a semantic clustering based deduction learning by mimicking the learning and thinking process of human brains. Human beings can make judgments based on experience and cognition, and as a result, no one would recognize an unknown animal as a car. Inspired by this observation, we propose to train deep learning models using the clustering prior that can guide the models to learn with the ability of semantic deducing and summarizing from classification attributes, such as a cat belonging to animals while a car pertaining to vehicles. 
The proposed approach realizes the high-level clustering in the semantic space, enabling the model to deduce the relations among various classes during the learning process. In addition, the paper introduces a semantic prior based random search for the opposite labels to ensure the smooth distribution of the clustering and the robustness of the classifiers. The proposed approach is supported theoretically and empirically through extensive experiments. We compare the performance across state-of-the-art classifiers on popular benchmarks, and the generalization ability is verified by adding noisy labeling to the datasets. Experimental results demonstrate the superiority of the proposed approach.
\end{abstract}


\end{frontmatter}

\nolinenumbers

\section{Introduction}
The powerful ability for feature expression and semantic extraction of deep Convolutional Neural Networks (CNNs) has dramatically pushed the flourishing development of computer vision~\cite{huang2017densely}~\cite{he2016deep}~\cite{he2018learning}. 
At the same time, large-scale labeled data samples ensure the effectiveness of supervised learning, which enables the deep learning models to efficiently extract abstract but highly-semantic information for complicated vision tasks~\cite{ma2020mdfn}~\cite{xu2019toward}~\cite{zhang2020self}. Undoubtedly, future learning models should be complex, robust, knowledge-driven, and cognition-based~\cite{marcus2020next}~\cite{cen2020deep}. This defines them with the cognitive ability of self-enhancing, synthesizing knowledge from multiple sources, and deducing based on knowledge and experiences~\cite{marcus2020next}.

\begin{figure*}[t]
    \includegraphics[width=1.0\linewidth]{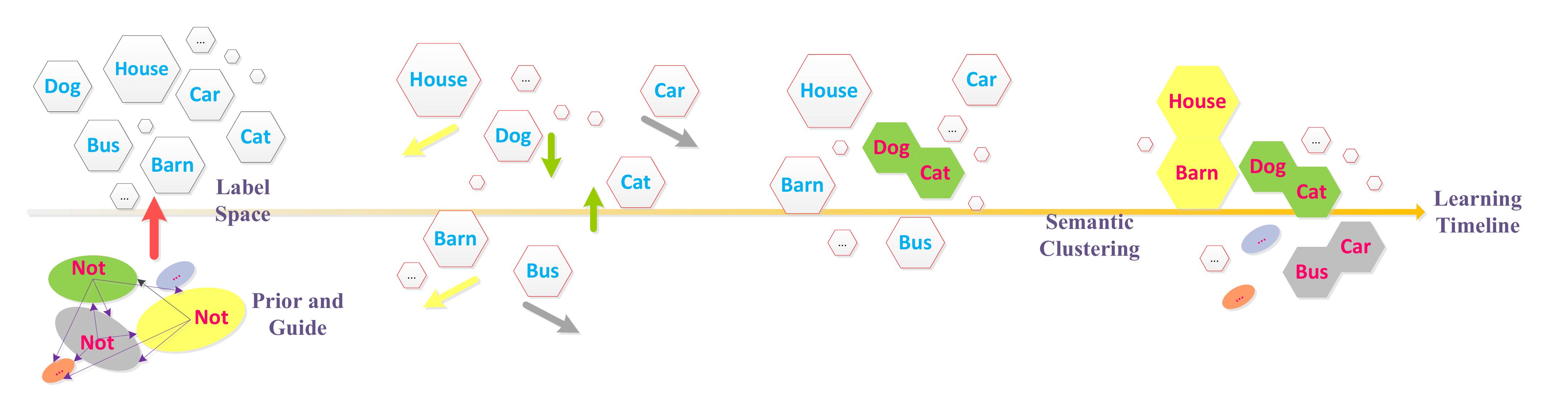}\hfill
    \vspace{-2mm}
    \caption{The deduction progress of semantic clustering. Prior and Guide works as the prior information that is combined with the labels as the labeling input data. The Learning Timeline is the same as the normal classification learning process, but our model provides the possibility of doing high-level semantic clustering by the deduction progress as the aid for the classification task. The model, at the end of the learning timeline, is expected to provide better classification accuracy.}
	\label{fig:conception}
 	\vspace{-3mm}
\end{figure*}

Some complementary and weak supervision information has been exploited to boost the learning performance of models~\cite{li2021sgnet}~\cite{ma2020location}. Such complementary supervision includes early side information\cite{xing2003distance}, privileged information~\cite{vapnik2009new}, and weak supervision based on semi-supervised data~\cite{law2017efficient}~\cite{haeusser2017learning}, noisy labeled data~\cite{gong2017learning}~\cite{misra2016seeing}, and complementary labels~\cite{ishida2017learning}~\cite{yu2018learning}~\cite{kim2019nlnl}. Most of these methods supplement extra direct labeling information or replace expensive accurate labels with cheap labeling information. These complementary labels, in fact, increase the labeling cost as a direct mapping from label space to sample space, named as ``hard labeling" in the later sections. Most importantly, these methods are unable to equip deep models with the ability of self-enhancement, synthesis, and deduction. 

In this paper, we leverage the wide-applied but fundamental supervised image classification and propose deduction learning by semantic clustering. 
We introduce semantic prior, high-level clustering information, represented by different colors in Figure~\ref{fig:conception}, although no names are given for each color. For example, we expect the model know the cat and the dog should be closed to each other, though the model would never know they should be called ``animal". Semantic prior (Prior and Guide to Label Space in Figure~\ref{fig:conception}) is thus introduced into the classification learning models, guiding them to form effective semantic clustering so that they are able to deduce high-level semantic expression (\eg Same color cells go attached together in Figure~\ref{fig:conception}), as shown in Figure~\ref{fig:conception}. 

Inspired by the idea of negative learning~\cite{kim2019nlnl}~\cite{yu2018learning}, we propose to guide the classifier to learn the opposite class that does not belong to the same cluster with the accurate label. For example, if a sample is labeled by ``cat", then our algorithm will tell the classifier that the image is not ``car" or any other random label that belongs to a different cluster other than ``cat", during one learning shot. 

This random search for the opposite label is in accordance with the semantic prior that is fed into the model along with other inputs that specifically refer to the images and their corresponding labels in this work. Statistically, the opposite semantic labels corresponding to a certain accurate label should be chosen with equal probability given the number of learning periods (epochs) is large enough. Theoretically, this proposed method enables a smooth clustering in the semantic space and an effective deduction, which makes the model able to deduce that ``cat" should be one element of an abstract cluster, although the model would never know it can be called ``animal", as shown in the second stage of Figure~\ref{fig:conception}, where the colors ``Green, Grey, Yellow", each represents a higher hierarchical category. Each specific class, like the cat, would be learning that if it belongs to ``Green", then it is totally on the opposite side of other classes that belong to ``Yellow" and ``Grey".   

Finally, it is noticed that the proposed method does not give up the conventional label learning by introducing one composite loss function. This ensures the label learning and the semantic clustering in the same timeline during the learning process. It conforms to the requirement of cognition learning~\cite{marcus2020next}. By this setting, the model could finish high-level semantic expressions, capturing the concepts, similar to ``animal", ``vehicle", ``buildings", etc., as shown in the third stage in Figure~\ref{fig:conception}, where sample classes accomplish clusters. 

The major contributions of this paper are summarized below: 
\begin{itemize}
    \item {\em Semantic Clustering:} We propose a high-level semantic mapping within semantic space, enhancing the semantic expression and providing a certain level of independence for overcoming the limitation of convolution operation at the pixel level. It is realized by introducing a semantic prior which could guide the model to find the opposite semantic label that is not from the same semantic colony with the given true label. 
    \item {\em Deduction Learning:} Deduction learning is realized by the semantic prior and the proposed random search for opposite semantic, which ensures the smoothness of semantic clustering and the robustness of classification. It could be implemented as a plug-in module that could play in arbitrary classification models by introducing a composite loss function.
    \item {\em Robust Improvement:} We achieved stable convergence and robust classification performance on mainstream classification models. It is also verified by working on noisy data environment where there exists a certain ratio of incorrect labels.   
    \item {\em Wide Applicability:} In the proposed method, label learning and semantic clustering follow the same learning timeline, equipping the model with the ability of deduction and cognition. It can be taken as a plug-in module for broad deep learning applications, such as  few-shot learning, zero-shot learning or even semi-supervised learning.
\end{itemize}

The functional source code of the paper can be accessed from the link \text{https://github.com/rucv/deduction-learning}.

\section{Related Work}
\subsection{Hierarchical Semantic Information} 
At first, the research in this field focuses on exploring or utilizing the inherent relations among label classes, or looking for the intermediate representations between classes. ~\cite{akata2015label} formed a label-embedding problem where each class is embedded in the space of attribute vectors so that the attributes act as intermediate representations that enable parameter sharing between classes. Another research in~\cite{deng2014large} uses a label relation graph to encode flexible relations between class labels by building the rich structure of real-world labels. The idea of incremental learning by hierarchical label training has been explored recently by a few other papers. Progressive Neural Networks~\cite{rusu2016progressive} learn to solve complex sequences of task by leveraging prior knowledge with lateral connections. “iCaRL” allows learning in a class incremental way: only the training data for a small number of classes is present at the same time and new classes
can be added progressively~\cite{rebuffi2017icarl}. Tree-CNN~\cite{roy2020tree}, proposes training root network by general classes and then learning the fine classes by corresponding growth-network (mainly learned by leaf structure of the network). While this research direction solves hierarchical semantic learning based on an independent timeline for each stage. Our proposed idea shares the same timeline with the normal classification task throughout the entire learning process which works as an exploration towards cognitive learning. At the same time, the methods above directly provide concrete class relation structure on the basis of the original class labels for training, without exploring the deduction ability of the networks.

Learning with real, concrete complementary labeling information was proposed by~\cite{ishida2017learning} for the image classification task. It was based on an assumption that the transition probability for complementary labels is equal to each other. 
It modified the traditional one-versus-all (OVA) and pairwise-comparison (PC) losses so that it is suitable for the uniform probability distribution, working as an unbiased estimator for the expected risk of true-labeled classification. Later on, the work~\cite{yu2018learning} argued that there are two unsolved problems in the previous work. The first one lies in the fact that the complementary labels tend to be affected by annotators' experience and limited cognition. The other one is the proposed modified OVA and PC losses can not be generalized to more popular losses, such as the cross-entropy loss. Thus, they proposed the transition matrix setting to fix the bias from the biased complementary labels. At the same time, they provided intensive mathematical analysis to prove their proposed setting can be generalized to many losses which directly provides an unbiased estimator for minimizing expectation risk. These works expect better semantic learning by introducing intensive complementary labeling while they do not explore the deduction ability of the networks themselves as well. They are essentially regular label learning. The work in~\cite{kim2019nlnl} automatically generated complementary labels from the given noisy labels and utilized them for the proposed negative learning, incorporating the complementary labeling into noisy label learning.

\subsection{Semantic Labeling in Noisy Cases} 
Some researchers attempt to aid learning in noisy cases by introducing effective semantic label learning. Some attempt to create noise-robust losses by introducing transition probabilities to the field of classification and transfer learning~
\cite{ghosh2017robust}~\cite{zhang2018generalized}. Some propose to use the transition layer to modify deep neural network~\cite{hendrycks2018using}. In other studies, researchers try to re-weight the training sample based on the reliability of the given label~\cite{ren2018learning}~\cite{lee2018cleannet}. Some other approaches try to prune the correct samples from the softmax outputs~\cite{ding2018semi}~\cite{tanaka2018joint}. Different from them, this paper dedicates to the research on how self-clustering and deduction learning ability of networks would influence the robustness in noisy labeling cases. 

This paper tries to explore the self-deduction ability of networks in the semantic space and focuses on guiding the models to fetch effective hierarchical semantic information in a self-learning way by semantic clustering and cognitive accumulation. First, it could completely free the confinement problem of transition probabilities. The proposed semantic prior based random search for opposite semantic ensures the equal probability, providing the mapping independence in semantic space. Second, the semantic clustering boosts positive label learning. For example, if the sample ``cat" has a low classification probability, the semantic clustering could help enhance this confidence by guiding this model to realize that the object is at least an animal, not a ``car". Third, our proposed method shares the same timeline with conventional label learning, enabling effective cognitive accumulation. Moreover, there is no need for specifically defining loss functions for the proposed models. Following the loss formations of the original label learning in specific models is all we need, potentially leading to better generalization.

\section{Problem Setup}
People can make deduction independent of the actual vision behavior. Thus, in deep learning, we expect the model with similar independence to ensure the realization of high-level mapping in semantic space.

\paragraph{Semantic Space for Image Classification}
Semantic space is originally proposed in the natural language domain, aiming to create representations of natural language that are capable of capturing meanings~\cite{baroni2010distributional}. In computer vision, the concept of semantic space is much more abstract. Current semantic extraction is limited both by spatial size and by the individual data sample. However, it should aim to overcome the limitations of convolution-based or receptive-field based approaches operating at the pixel level. Convolution-based deep learning models are fixed at the pixel level and are poor for generalization, which would easily break down if the individual image differs from or is strange to those in the training materials used for the statistical models. Compared to spatial feature learning that performs at the pixel level, semantic learning should be a relatively independent process that works on the semantic element, which is the common description for a class of objects. Moreover, the semantic expression could have multi-levels that describe the relevant or diverging characters of semantic elements. For example, the ``cat" as a  semantic element could be clustered to the high-level semantic expression, something similar to an ``animal". 

\begin{definition}[Semantic Space]\label{lem:1}
	Without loss of generality, let $\mathcal{C}$ be the semantic space, $c\in\mathbb{Z^{+}}$ be the semantic element in $\mathcal{C}$ that appears as one semantic label indicating a specific object class. The semantic relation of different  $c$ is defined by $r$. [c] = \{{1,...,c}\} signifies the set of semantic labels. 
	Then, we have
	\begin{align}
		\mathcal{C} & \stackrel{def}{=} \langle [c], r \rangle
	\end{align}
	where element $c$ is uniformly sampled from $\mathcal{C}$. Tuple $\langle [c], r \rangle$ expresses the fact that semantic elements $c\in[c]$ are linked to each other by the relation $r$, forming the abstract spatial distribution in $\mathcal{C}$.  
\end{definition}

\paragraph{Semantic Cell}
In order to better describe the abstract relation distribution in $\mathcal{C}$, we propose {\it Semantic Cell} as the semantic unit that could label a group of objects that have similar features in feature space $\mathcal{X}$, which corresponds to the element $c\in[c]$ in Definition \ref{lem:1}. It realizes a multi-to-one mapping that bridges the link between feature space $\mathcal{X}$ and semantic space $\mathcal{C}$.
\begin{definition}[Semantic Mapping]\label{lem:2}
Let $g(\mathbf{x})$ be the mapping function of a given multi-class classification learning model that estimates the classification probabilities based on the input sample $\mathbf{x}$ in feature space $\mathcal{X}$. $f(\mathbf{x})$ predicts the classification label $y$ based on the maximum probability principle, mapping the feature sample $\mathbf{x}$ to the corresponding semantic cell $c$ in $\mathcal{C}$. 
\begin{align}\label{eqn:R}
f(\mathbf{x}) \stackrel{def}{=} \arg \max_{i\in[c]} g_{i}(\mathbf{x}) 
\end{align}
\end{definition}
\noindent where $f: \mathcal{X} \rightarrow \mathcal{C}$, the maximum probability of $g$ and $f(\mathbf{x})\in \mathcal{C}$. $g_{i}(\mathbf{x})$ realizes the estimation towards $P(y=i|\mathbf{x})$. 
\paragraph{Semantic Colony}
Semantic Colony $\theta$ takes semantic cell $c$ as individual sample. It clusters $c\in \mathcal{C}$ that hold related semantic information as $\theta$. Based on which, it defines the intra-class relation and inner-class differentiation to realize clustering in semantic space $\mathcal{C}$ with high-order semantic expression. 
\begin{definition}[Semantic Clustering]\label{lem:3}
	Without loss of generality, let $\Theta$ be the distribution of semantic colonies $\theta$ in $\mathcal{C}$. $H$ conducts clustering for semantic cell $c\in \mathcal{C}$ into semantic colony $\theta\sim\Theta$. $\mathbf{c}$ is the vector with the elements of semantic cells $c\in[c]$. 
	Then, we have
	\begin{align}
	    \theta \stackrel{def}{=} H(\mathbf{c}, r_{\mathbf{c}}) 
	\end{align}
	where $H: [c]\rightarrow\Theta$, $\mathbf{c}$ consists of semantic cells $c$ in [c] that are semantically related, and $H$ maps $\mathbf{c}$ to $\theta\sim\Theta$ in accordance with the corresponding semantic relation $r_{\mathbf{c}}$. 
\end{definition}

\section{Methodology}
In this section, we first introduce the general approach that deep neural networks learn optimal classification with hard labels. Then, we discuss the learning with semantic deduction and propose corresponding training and test model. 
\subsection{Conventional Classification Learning}
In multi-class classification, we aim to learn a classifier $f(\mathbf{x})$ that predicts the classification label $y$ for a given observation sample $\mathbf{x}$. Typically, the classifier directly maps $\mathbf{x}$ into the label space $\mathcal{Y}$ by the following function:
\begin{equation}\label{eq4}
    f(\mathbf{x}) = \arg \max_{i\in\mathcal{Y}} W_{i}^{T}\mathbf{x}
\end{equation}
where $f: \mathcal{X} \rightarrow \mathcal{Y}$ and $W_{i}$ refers to the learning parameters of the classifier $f$, with the estimation of $P(y=i|\mathbf{x})$. 

In supervised learning, loss functions are proposed to measure the expectation of the predicting $f(\mathbf{x})$ for $y$~\cite{bartlett2006convexity}. It is typically defined as the expected risk \cite{yu2018learning} for various loss functions. 
\begin{equation}
R(g) = \mathbb{E}_{\mathbf{x},y \sim P(\mathbf{x},y)}[\ell(f(\mathbf{x}),y)]    
\end{equation}

A well-trained classifier $f^{*}$ minimizes this expected risk $R(g)$,
\begin{equation}
f^{*} = \arg \min_{f\in\mathcal{F}} R(f)   
\end{equation}
where $\mathcal{F}$ is the distribution space of $f$.

\subsection{Learning with Semantic Deduction}
In semantic space, the description of hard labels towards objects is limited. To better describe an object or a scene, people usually enumerate related features and associate their prior cognition and experience for a reasonable deduction. Current deep learning models realize feature sensing and learning but lack the proper deduction that could enrich the description of objects. Our previous analysis shows that hard labels in semantic space could potentially build more links, as the discussion in Section 2.1. We introduce the semantic prior, guiding the model to learn the semantic links by deduction. The overview of our method is depicted in Figure~\ref{fig:overview}. 
The overall inputs include training sample images, corresponding labels, and the semantic prior information which provides the high-level semantic hierarchy of current classification labels. The classification model is trained in the same way as the original network. For the green part in Figure~\ref{fig:overview}, given label $y$, the model finds the corresponding opposite semantic label for the sample image according to the semantic prior by an equal-probability random search, shown as the yellow block. Then both the true label and the opposite semantic label are fed into the composite loss we defined. The output of the proposed method is expected of better classification performance in the way of classification accuracy. 

First, the semantic prior works as the criterion for colonies' formation in semantic space $\mathcal{C}$. For example, a cat labeled by $c_{i}\in[c]$ should be grouped into ``animal" colony, if denoted by $\theta_{m}$. Similarly, a car labeled by $c_{j}$ could be grouped into the ``vehicle" colony $\theta_{n}$.  
Second, the semantic deduction is fully performed in semantic space $\mathcal{C}$, instead of defining complementary labels as weak supervision. Thus, we do not need any tedious and laborious labeling work, which would avoid labeling bias from human beings' bias~\cite{yu2018learning}, and the problem that the complementary labeling is essentially non-uniformly selected from the $c-1$ classes other than the true label class ($c>2$). 

\subsection{Equal-Probability Search for Opposite Semantic.}
We assume that the variables $(\mathbf{x}, c, \theta)$ are defined in the space $(\mathcal{X}\times[c]\times\Theta)$, with the joint probability measure $P(\mathbf{x}, c, \theta)$. 
\begin{center}
\begin{figure*}
    \includegraphics[width=1.0\linewidth]{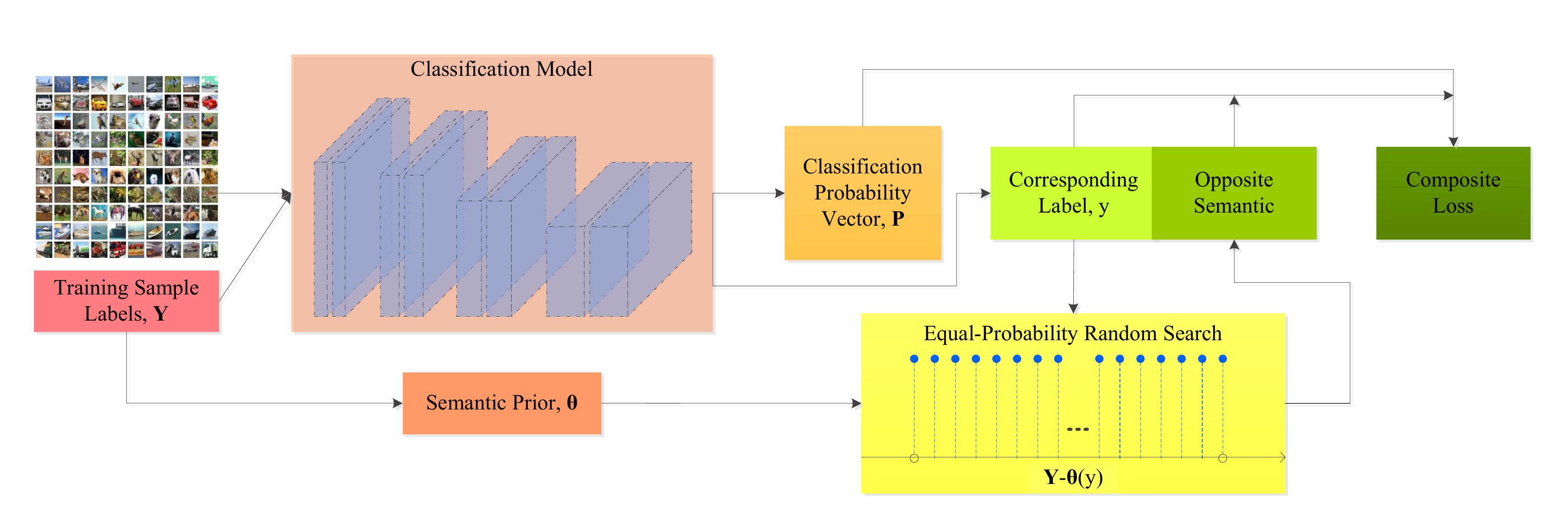}\hfill
    \vspace{-2mm}
    \caption{An overview of the proposed method. We use semantic prior based random search to produce opposite semantic so as to form the composite loss function, guiding the model to form semantic colonies. }
	\label{fig:overview}
 	\vspace{-3mm}
\end{figure*}
\end{center}

Given a sample $(\mathbf{x},c, \theta)\in(\mathcal{X}\times[c]\times\Theta)$, its opposite classification label $\bar{c}$ is randomly selected from $[c]\backslash\theta$. When the sampling frequency in a complete learning period is greatly larger than the class number $n_{[c]}$, the probability for each $\bar{c}\in[c]\backslash\theta$ that indicates how likely it is selected can be expressed as
\begin{equation}
P_{i}(\bar{Y}=\bar{c}|X=\mathbf{x}, Y=c) = \frac{1}{n_{([c]\backslash\theta)}}
\end{equation}
where $n_{([c]\backslash\theta)}$ is the number of semantic cells in $[c]\backslash\theta$. This conclusion verifies that the proposed semantic-prior based random search method for the opposite semantic label $\bar{c}$ is statistically consistent, 
and it realizes the independency of $\bar{c}$ with respective to feature space $\mathcal{X}$ conditioned on $c$ and $\theta$. Thus we have,
\begin{equation}
    P(\bar{Y}=\bar{c}|X=\mathbf{x}, Y=c) = P(\bar{Y}=\bar{c}|Y=c)
\end{equation}

The optimal classifier can be found under the uniform assumption, which has been proven in previous work~\cite{ishida2017learning}.
Meanwhile, the uniform selection means equal probability, ensuring the smooth clustering and the stability and robustness of the learning process. While for man-made complementary labels, they are confined by the fact that $\bar{Y}$ is assumed to be independent of feature $\mathcal{X}$~\cite{yu2018learning}~\cite{ishida2017learning}. 

Based on the exist of independence, the complete mapping from $\mathbf{x}$ to $\bar{y}$ can be set up as the following formula, $\forall i,j\in[c]$, 
\begin{equation}\label{eq9}
\begin{aligned}
    P(\bar{y} | \mathbf{x}) &= \sum_{i\in\theta_{i}, j\notin\theta_{i}} P(\bar{y}=j, y=i | \mathbf{x}) \\
    &= \sum_{i\in\theta_{i}, j\notin\theta_{i}} P(\bar{y}=j|y=i,\mathbf{x}) P(y=i|\mathbf{x})\\
    &= \sum_{i\in\theta_{i}, j\notin\theta_{i}} P(\bar{y}=j|y=i) P(y=i|\mathbf{x})
\end{aligned}
\end{equation}

\subsection{Learning with Smooth Semantic Clustering}
Conventionally, the classifier is trained to learn that the input image belongs to a specific, single class label. Let $\mathbf{x}\in\mathcal{X}$ be the input image, $y\in[c]$  denotes its label. $f(\mathbf{x}, W)$ maps the input $\mathbf{x}$ to the score space: $\mathcal{X}\rightarrow\mathbb{R}^{c}$, as equation \eqref{eq4} shows. The training process is guided by the cross entropy loss (most popular classification cost function) of $f$ as
\begin{equation}
    \mathcal{L}_{\mathbb{P}}(f, y) = -\sum_{m=1}^{c}\mathbf{y}_{m} \log \mathbf{p}_{m}
\end{equation}
where $\mathbf{y}\in\{0,1\}^{c}$ is the one-hot vector form of $y$. $\mathbf{p}_{m}$ is the $m^{th}$ element of probability vector $\mathbf{p}$. The conventional learning process is to optimize the probability $\mathbf{p}_{m}$ according to the given exact label $\mathbf{y}_{m}$ so that $\mathbf{p}_{m}\rightarrow1$. Based on which, we propose a learning algorithm with smooth high-level clustering by guiding $f$ to learn the semantic prior from the opposite label. 
Inspired by~\cite{kim2019nlnl}, the opposite semantic should push $f$ to optimize the corresponding classification probability $\bar{\mathbf{p}}_{m}\rightarrow0$. 
\begin{equation}
    \mathcal{L}_{\mathbb{O}}(f, y) = -\sum_{m=1}^{c}\bar{\mathbf{y}}_{m} \log (1-\bar{\mathbf{p}}_{m})
\end{equation}
where $\mathbf{y}_{m}\in\theta_{m}$,  $\bar{\mathbf{y}}_{m}\in[c]$ and $\bar{\mathbf{y}}_{m} \notin\theta_{m}$. $\bar{\mathbf{p}}_{m}$ is the corresponding classification possibility of label $\bar{\mathbf{y}}_{m}$ in vector $\mathbf{p}$. Thus, the random selection of $\bar{\mathbf{y}}_{m}$ comes from $[c]\backslash\theta$ in every iteration during the training process, shown in Algorithm 1. 
\begin{algorithm}
\caption{Smooth Semantic Clustering}
\begin{algorithmic}[1]
    \INPUT Training label $y\in\mathcal{Y}=[c]$, semantic prior $\hat{\theta}\sim\hat{\Theta}$
    \While{iteration}  
        \If{$y\in\hat{\theta}_{i}$}
            \State $\bar{y}$ = Select randomly from $[c]\backslash\hat{\theta}_{i}$
        \EndIf
        \State There exists another semantic colony $\theta_{j}$ 
        \If{$\bar{y}\in\theta_{j}$}
            \State $y\notin\theta_{j}$
        \EndIf
    \EndWhile  
    \OUTPUT{Opposite semantic label $\bar{y}$ and the learned semantic colony $\theta\sim\Theta$}
\end{algorithmic}
\end{algorithm}

From Algorithm 1, we can observe that the learning for clustering in the semantic space $\mathcal{C}$ is synchronous with image classification. Thus, we can define a composite loss function for an end-to-end semantic clustering classifier. 
\begin{equation} \label{eq:loss}
\begin{aligned}
   \mathcal{L} &= \alpha_{1}\mathcal{L}_{\mathbb{P}} + \alpha_{2}\mathcal{L}_{\mathbb{O}} \\
              &=-\alpha_{1}\sum_{m=1}^{c}\mathbf{y}_{m} \log \mathbf{p}_{m} -\alpha_{2}\sum_{m=1}^{c}\bar{\mathbf{y}}_{m} \log (1-\bar{\mathbf{p}}_{m})
\end{aligned}              
\end{equation}
where $\alpha_{1}$ and $\alpha_{2}$ are weights defining the ratio of $\mathcal{L}_{\mathbb{P}}$ and $\mathcal{L}_{\mathbb{O}}$ respectively.

For a specific input image, there is not only a semantic label $y$ but also other semantic description $\theta\sim\Theta$, and $\theta$ is the high-level semantic expression corresponding to $y$, which builds a new semantic attribute with a larger range. 
Since the opposite semantic is randomly selected with equal probability, the clustering hyperplane in $\mathcal{C}$ can be smooth. 

\subsection{Optimal Learning}
In the case of $\mathcal{L}$, we define the expected risk $\bar{R}(f)$ with the mapping $f:\mathcal{X}\rightarrow\{[c], \Theta\}$. If we can find an optimal $f^{*}$ such that $f^{*}=P(Y=i|X), \forall i\in[c]$, then in theory, we expect that we can find the optimal $\bar{f}^{*}$ such that $\bar{f}^{*}=P(\bar{Y}=i|X),\forall i\in[c]$, where $P(\bar{Y}|X) = \sum_{i\in\theta_{i}, j\notin\theta_{i}} P(\bar{Y}=j,Y=i|X)$ according to equation \eqref{eq9}. If the above idea can be proved, with sufficient training samples, the proposed algorithm with $\bar{R}(f)$ is capable of simultaneously learning a good classification and clustering for $(X,Y,\theta)$. 

Following \cite{yu2018learning}, we will prove that the proposed semantic clustering learning with its corresponding loss function $\mathcal{L}$ is able to identify the optimal classifier. First, we introduce the following assumption~\cite{yu2018learning}, 
\begin{assumption}\label{as:1}
  The optimal learning with mapping $f^{*}$ satisfies $f_{i}^{*}(X)=P(Y=i|X), \forall i\in[c]$ by minimizing the expected risk $R(f)$.
\end{assumption}

Based on this assumption, we are able to prove that $\bar{f}^{*}=f^{*}$ following the theorem below ~\cite{yu2018learning}.
\begin{theorem}
Suppose that Assumption 1 is satisfied, then the minimum solution $\bar{f}^{*}$ of $\bar{R}(f)$ is also the minimum solution $f^{*}$ of $R(f)$, i.e., $\bar{f}^{*}$=$f^{*}$.
\end{theorem}
\begin{proof}
Based on Assumption 1, loss function $\mathcal{L}$, and function~\eqref{eq9} for the learning in the proposed smooth semantic clustering, we have
\begin{equation}
\begin{aligned}
   f^{*}_{i}(X) &= P(\bar{Y}=j|X) \\
   &=\sum_{i\in\theta_{i}}P(\bar{Y}=j,Y=i|X), \forall i,j\in[c], j\notin\theta_{i}
\end{aligned} 
\end{equation}
Let $\bar{\mathbf{s}}(X)=[P(\bar{Y}=1|X),\cdots,P(\bar{Y}=c)|X)]$ and $\mathbf{s}(X)=[P(Y=1|X),\cdots,
P(Y=c)|X)]$. According to the discussion of~\cite{yu2018learning}, we rewrite $\bar{R}(f)$ as
\begin{equation}
\begin{aligned}
    \bar{R}(f)&=\int_{X}\sum_{j=1}^{c}P(\bar{Y}=j)P(X|\bar{Y}=j)\mathcal{L}(f(X), \bar{Y}=j)dX \\
    &=\sum_{j=1}^{c}P(\bar{Y}=j)\int_{X}P(X|\bar{Y}=j)\mathcal{L}(f(X), \bar{Y}=j)dX \\
    &=\sum_{j=1}^{c}P(\bar{Y}=j)\bar{R}_{j}(f)
\end{aligned}    
\end{equation}
where $P(\bar{Y}=j)$ is given when we have $Y=i$, distributed as $P(\bar{Y}=j|Y=i)$ according to Algorithm 1. $\bar{R}_{j}(f)=\int_{X}P(X|\bar{Y}=j)\mathcal{L}(f(X), \bar{Y}=j)dX$. Thus, if we use $\mathbf{C}$ to denote the operation form of $P(\bar{Y}=j|Y=i)$, according to function~\eqref{eq9} and the above convergence analysis, we have
\begin{equation}\label{eq14}
    \mathbf{\bar{s}}(X)=\mathbf{C}^{T}\mathbf{s}(X)
\end{equation}
where $P(\bar{Y}=j|Y=i)$ is realized based on the random search with semantic prior. 
Equation~\eqref{eq14} ensures that
\begin{equation}
    \bar{f}^{*}(X)=\arg\max_{i}\mathbf{C}^{T}\mathbf{s}_{i}(X)=\mathbf{C}^{T}\arg\max_{i} \mathbf{s}_{i}(X) = \mathbf{C}^{T} f^{*}(X)
\end{equation}
where $i\in[1,c]$. Thus, we have $\bar{f}^{*} \Longleftrightarrow f^{*}$. The proof is completed.
\end{proof}

\section{Experiment}
In this section, we study the impact of the proposed semantic deduction algorithm on popular image classifiers using mainstream benchmark datasets. 
In order to show that our algorithm is able to generalize to complex or disordered data environment with better robustness, we follow each specific experimental setting of the baseline methods, and only vary the data environment by producing noisy labels at certain ratios. 

\paragraph{Learning Scenarios} To identify the gain of the proposed deduction learning algorithm, we design fairly comparable learning scenarios where only the deduction related hyper-parameters are changed from the default original setting while keeping all the rest unchanged. The assignment for the weights of $\alpha_{1}$ and $\alpha_{2}$ in equation\ref{eq:loss} is based on the experiment performance. We introduce the most core algorithm idea of the current state-of-the-art works of complementary supervision information designed for various fields~\cite{kim2019nlnl}~\cite{yu2018learning}~\cite{ishida2017learning} into our experiment setting as one of the baselines. Details are listed below:
\begin{itemize}
    \item {\em Default Setting (OT):} In this setting, we train the original baseline classification models and keep all the hyper-parameters unchanged as in the corresponding published papers and public code. We take both classical and state-of-the-art CNN classifier networks into consideration, including Multilayer Perceptron (MLP)\cite{nokland2019training}, VGG~\cite{nokland2019training}, ResNet\cite{he2016deep}, DenseNet~\cite{huang2017densely}, Wresnet~\cite{zagoruyko2016wide}, ResNext\cite{xie2017aggregated}. All of them are trained and compared with our proposed methods fairly.
    \item {\em Random Opposite Semantic (RT):} Under this setting, 
    we exploit the opposite semantic label $\bar{y}\in[c]$ that corresponds to the original accurate label $y\in[c]$, satisfying $\bar{y}\neq y$. We use random search for the opposite label in the label pools $[c]$~\cite{kim2019nlnl} instead of hard labeling so as to avoid bias \cite{yu2018learning}~\cite{kim2019nlnl}. Thus, this setting does not refer to the semantic prior when looking for the opposite semantic label $\bar{y}$. All other settings follow the Default Setting. 
    \item {\em Semantic Deduction (SD):} We implement the proposed deduction learning by semantic clustering. The opposite semantic label $\bar{y}$ is randomly selected from $[c]\backslash\hat{\theta}_{i}$, where $[c]$ is the set of semantic labels. ${\hat\theta}_{i}$ is the $i\_th$ semantic colony (details in Algorithm 1). Thus, it naturally satisfies $\bar{y}\neq y$, $y$ referring to the original accurate label $y\in[c]$. It strictly follows the training setting with the identical hyper-parameters to those in the Default Setting. 
\end{itemize}
\paragraph{Data Sets}
\begin{itemize}
    \item {\em Fashion-MNIST:} Fashion-MNIST is a new image classification benchmark with different data classes of clothing\cite{xiao2017fashion}. The dataset has an image size of 28$\times$28, input channels of $1$, and the number of classes of $10$. In our SD setting, we provide the semantic prior for it to group the $10$ classes fashion clothing into three groups: ``clothes", ``shoes", and ``bags". 
    
    \item {\em CIFAR10:} CIFAR10 consists of $50,000$ training images and $10,000$ test images of dimension $32\times32$. It has a total of $10$ general classes\cite{krizhevsky2009learning}. In the SD setting, we group the $10$ classes into two groups, ``vehicles" and ``animals". 
    \item {\em CIFAR100:} CIFAR100 has $50,000$ training images and $10,000$ test images of the resolution of $32\times32$. It has a total of $100$ classes, with $500$ training images in each class \cite{krizhevsky2009learning}. For the SD setting, we provide two schemes, ``SD\_v1" and ``SD\_v2. The former one divides classes into ``7" groups, including ``people", ``animal", ``man-made stuff", ``transportation", ``plants", ``building", and ``nature". The latter contains 8 groups: ``people", ``animal", ``life appliances", ``transportation", ``food", ``plants", ``building", and ``nature", isolating ``food" from the ``man-made" as an independent expression.
\end{itemize} 

\begin{figure*}[t]
    \subfigure{\includegraphics[width=0.5\linewidth]{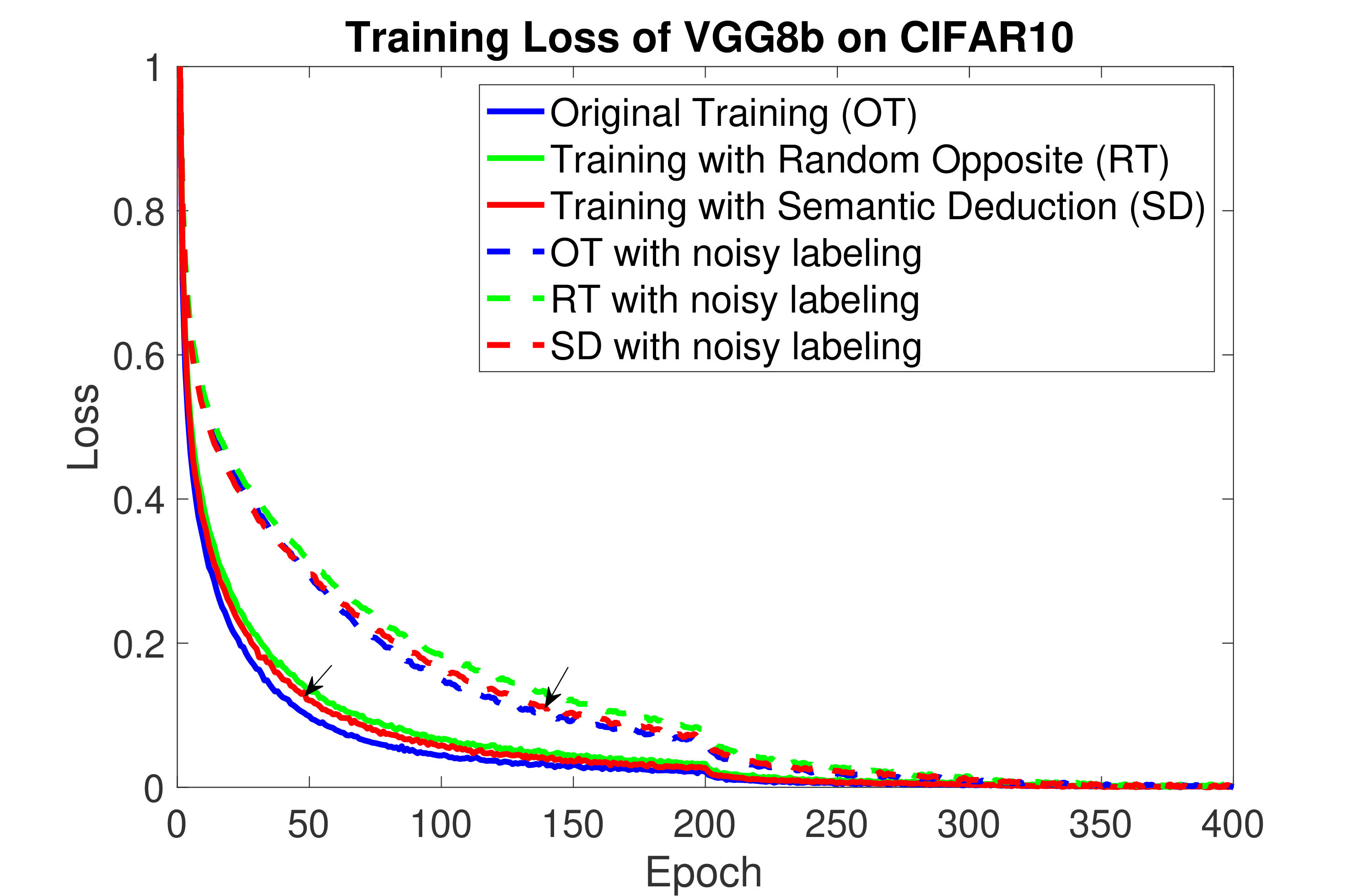}}\hfill
    \subfigure{\includegraphics[width=0.5\linewidth]{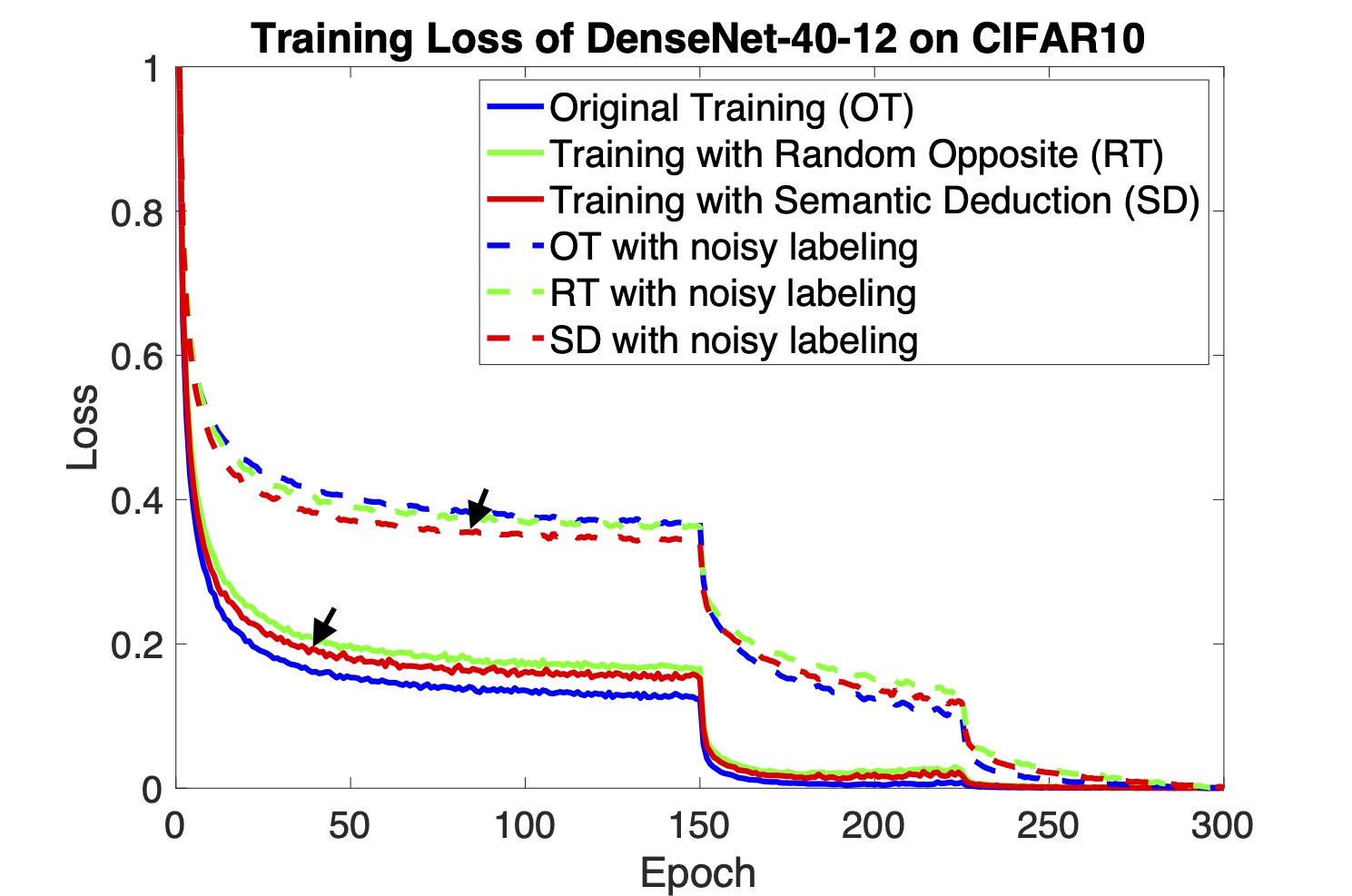}}\hfill
    \subfigure{\includegraphics[width=0.5\linewidth]{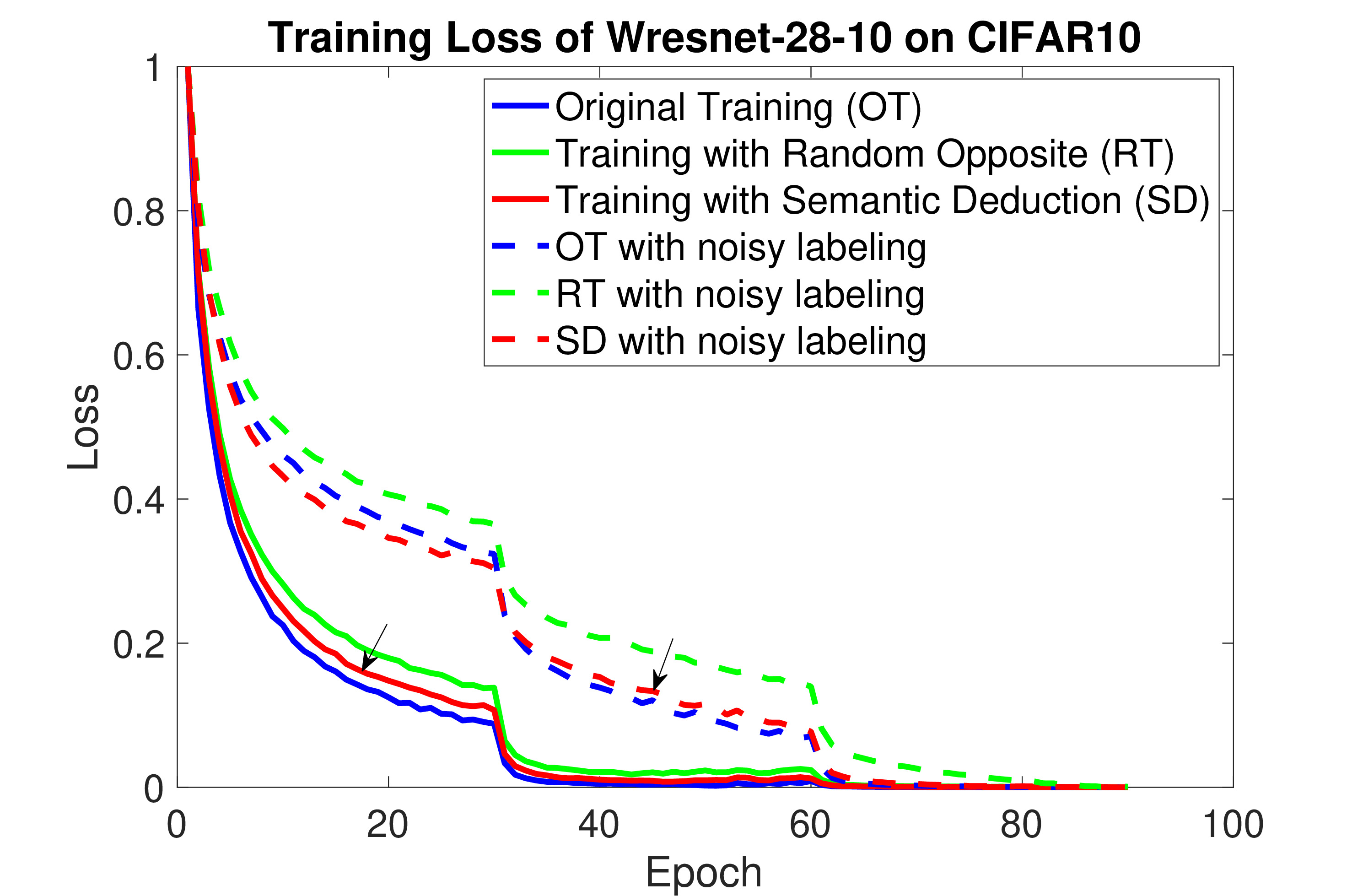}}\hfill   
    \subfigure{\includegraphics[width=0.5\linewidth]{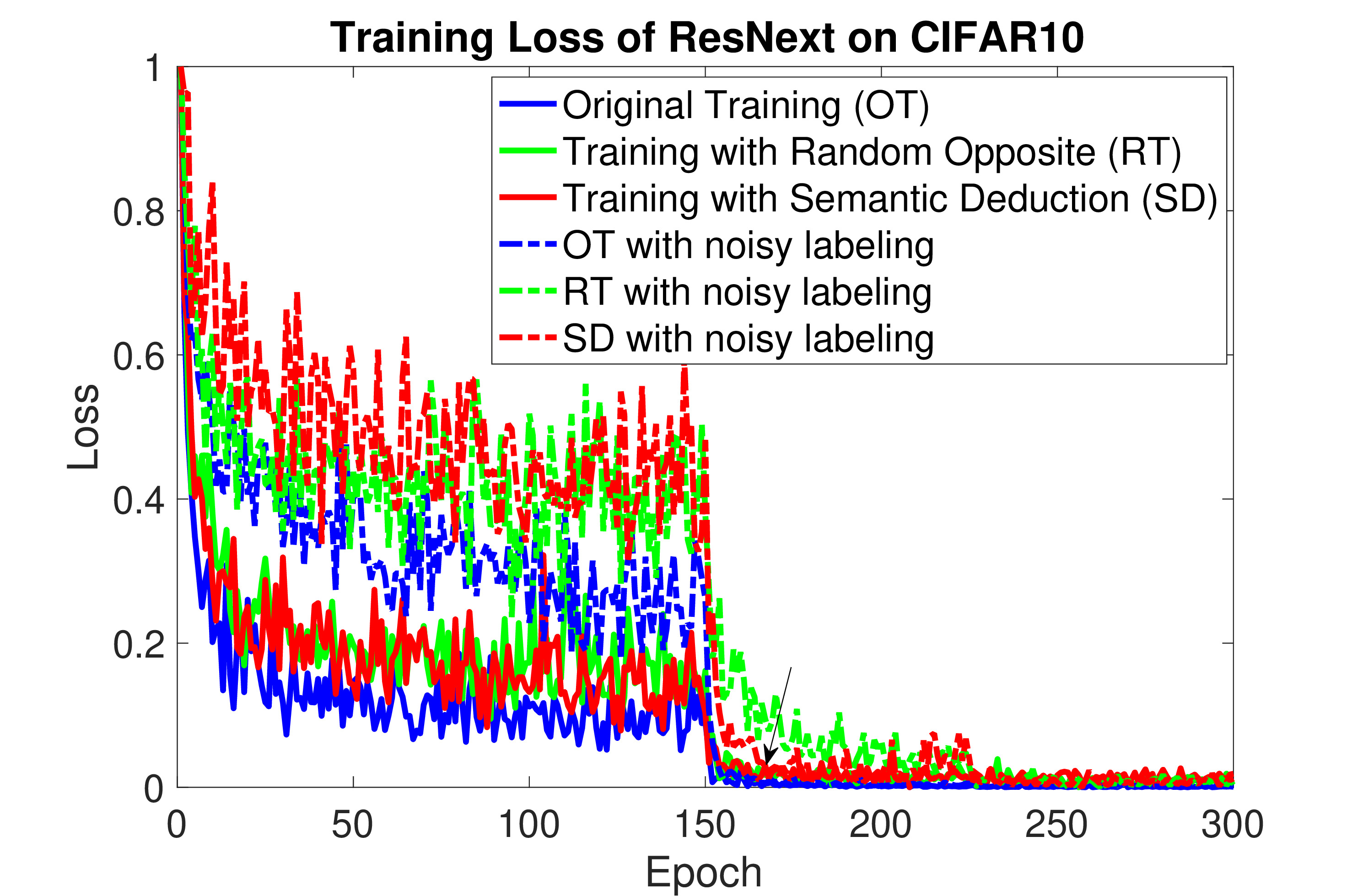}}\hfill
    \vspace{-2mm}
    \caption{Convergence performance of different models by training loss on CIFAR10.}
	\label{fig:cifar10}
\end{figure*}

\begin{figure*}
    \subfigure{\includegraphics[width=0.5\linewidth]{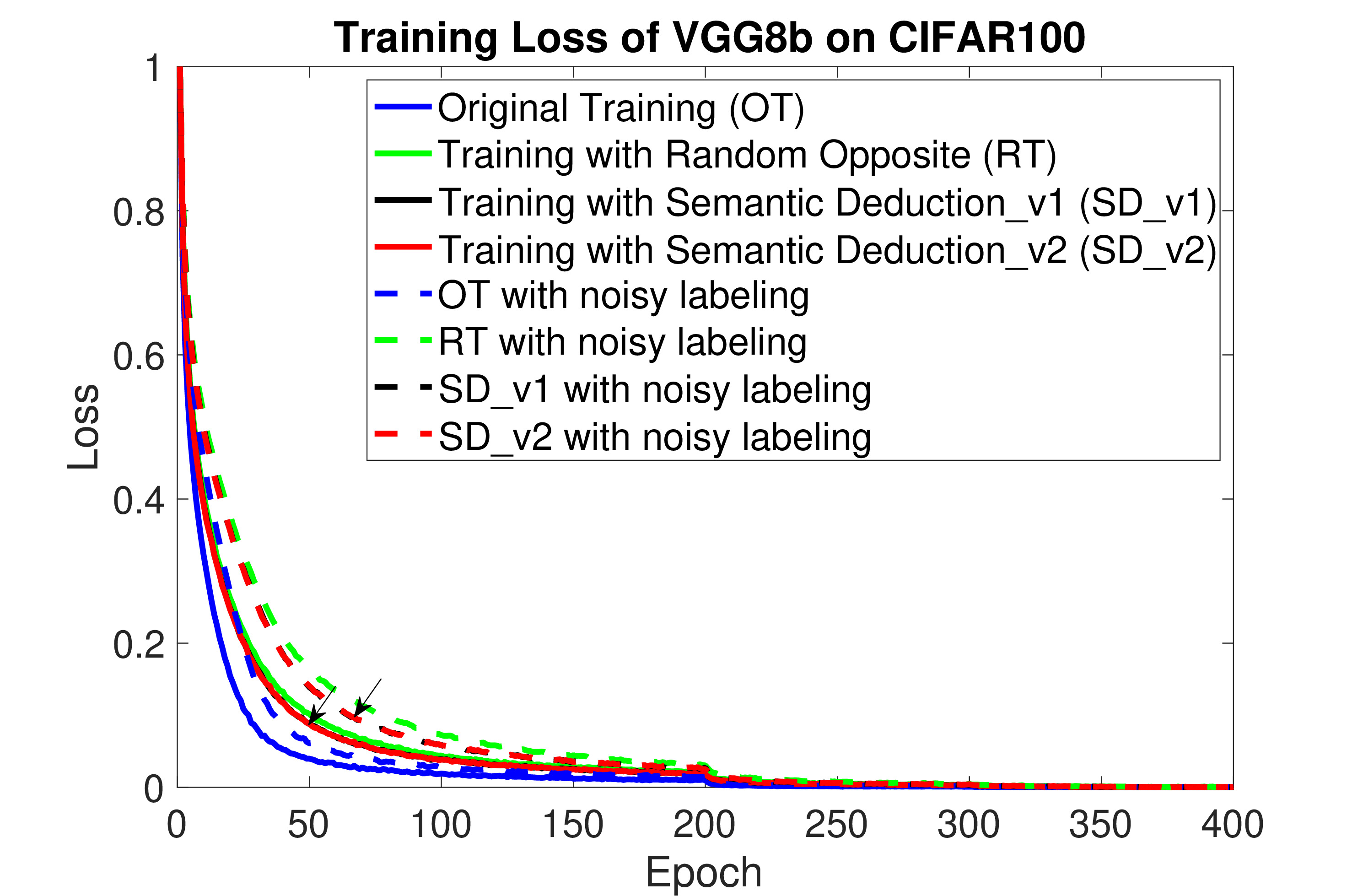}}\hfill
    \subfigure{\includegraphics[width=0.5\linewidth]{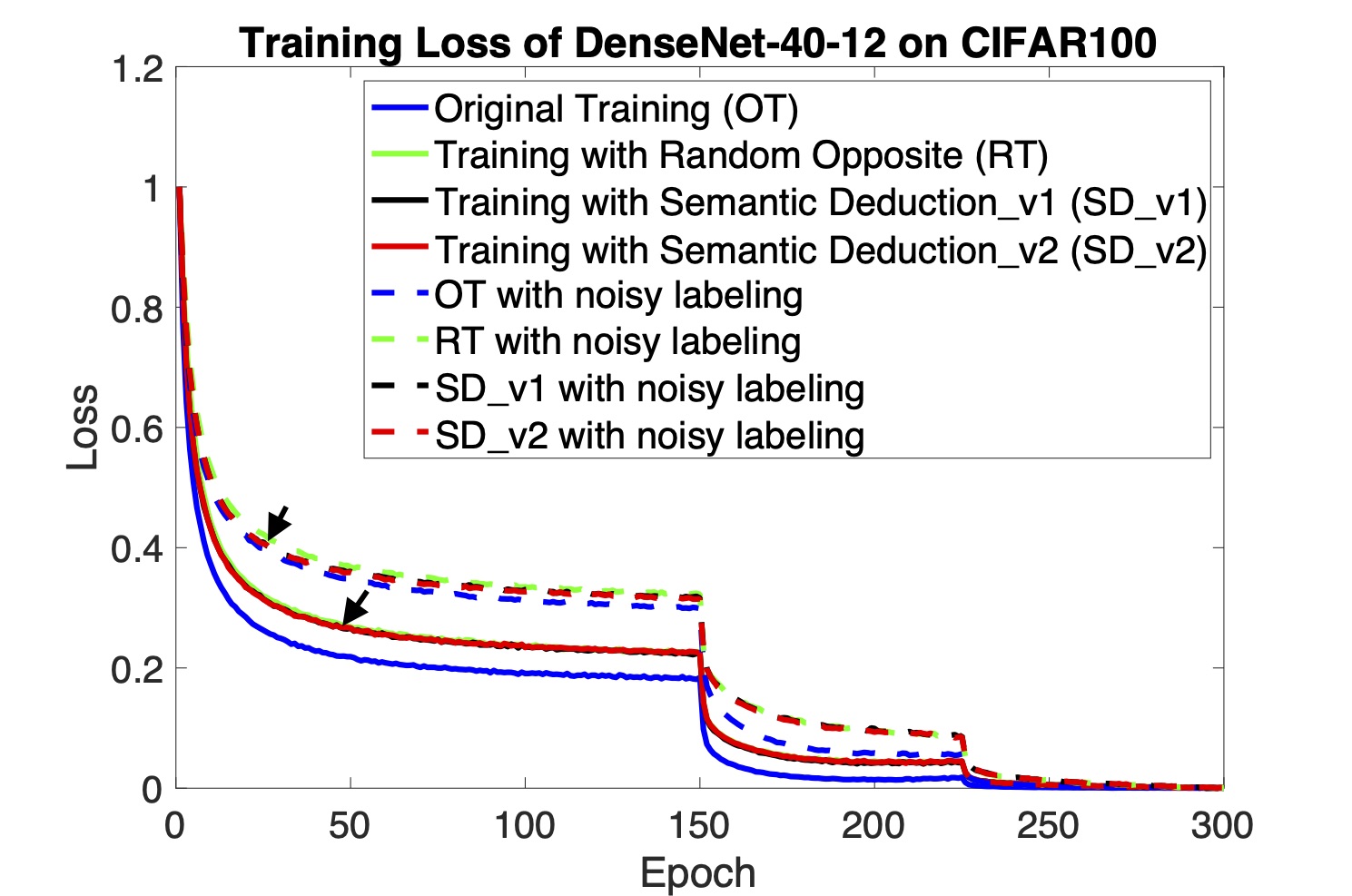}}\hfill
    \subfigure{\includegraphics[width=0.5\linewidth]{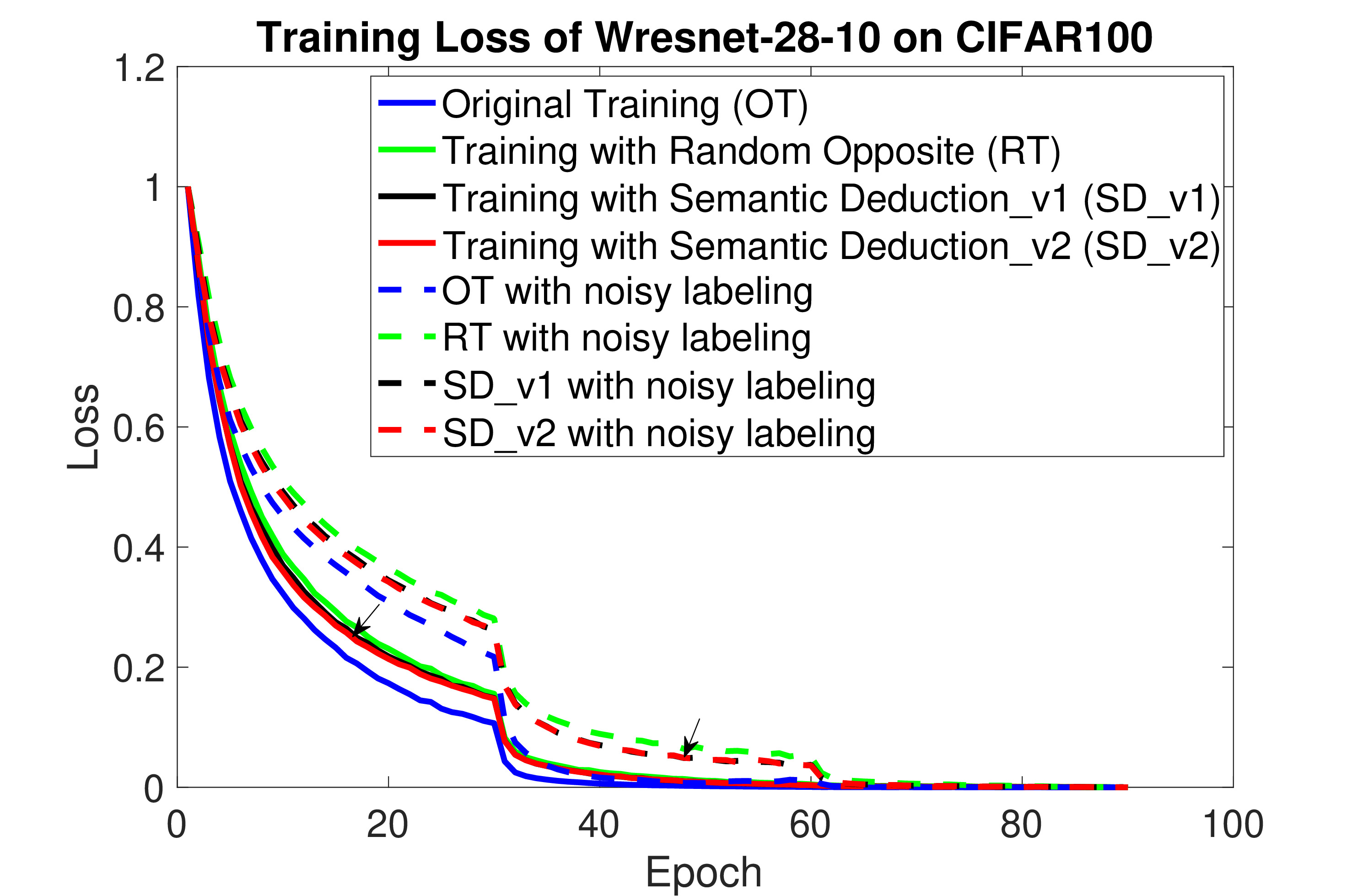}}\hfill    
    \subfigure{\includegraphics[width=0.5\linewidth]{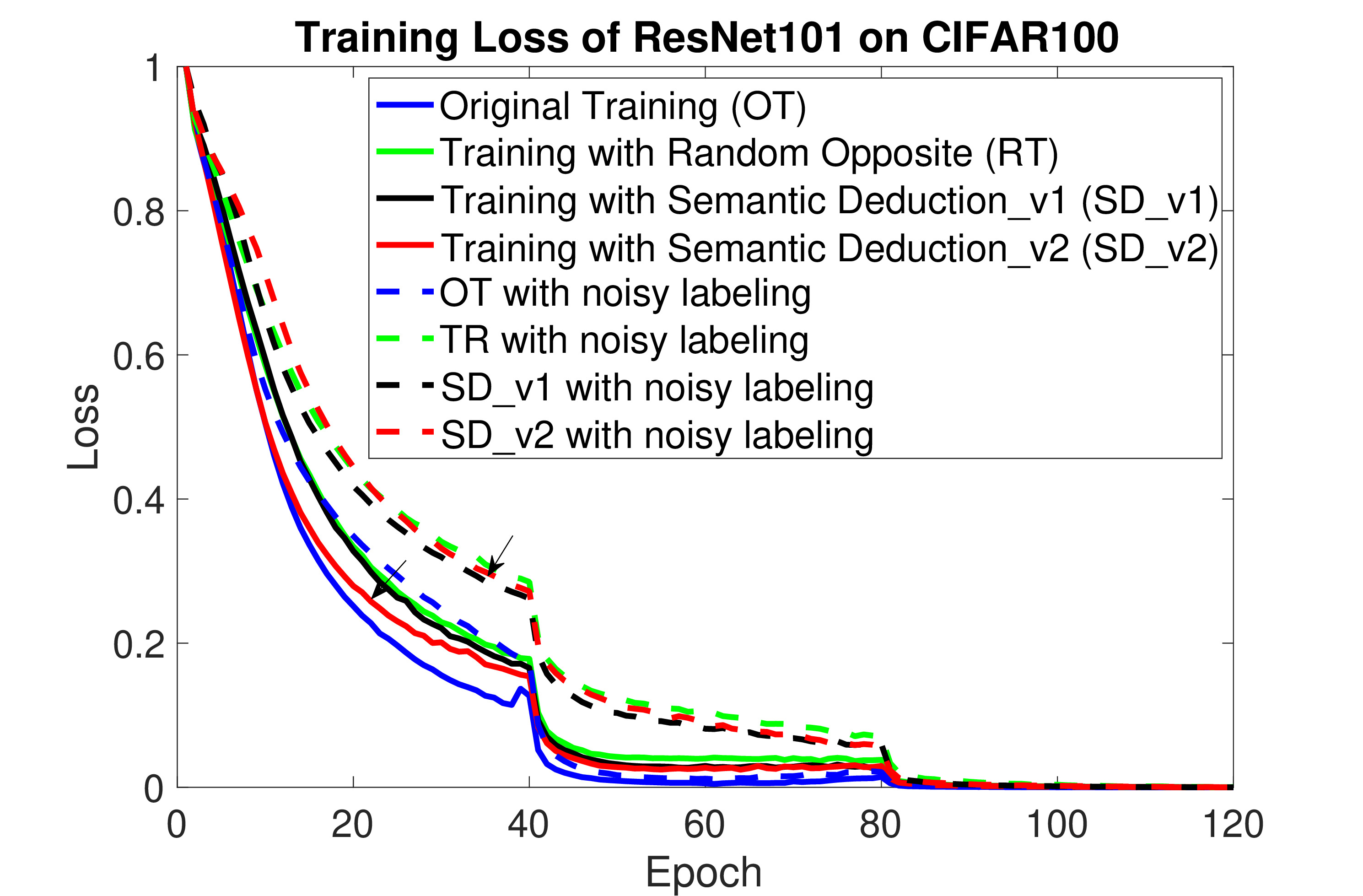}}\hfill
    \caption{Convergence performance comparison by training loss on CIFAR100.}
	\label{fig:cifar100}
\end{figure*}

\subsection{Results in Original Data Environment}
We first evaluate our proposed algorithm in the original data environment, directly using the images from the data sources. From the mathematical analysis in Section 4, the identification of optimal learning depends on stable convergence performance. Thus, we summarize the learning behaviors of each approach on CIFAR10 and CIFAR100 in Figure~\ref{fig:cifar10} and Figure~\ref{fig:cifar100}, respectively.

\paragraph{Convergence Performance} To obtain a fair comparison, we normalize the loss distribution to $[0,1]$ for all scenarios. \textbf{(a)} Our algorithm generally shows consistent convergence with different classifiers, as shown in the red or yellow solid lines in Figure~\ref{fig:cifar10} and~\ref{fig:cifar100}. We can see that SD usually converges faster than RT as the black arrows shown in almost every case. This consistent performance verifies that the proposed self-clustering learning process helps speed up convergence, assisting the classifier to execute the right decision, although there is no additional labeling information fed into these models. \textbf{(b)} From all the sub-figures in both Figure~\ref{fig:cifar10} and~\ref{fig:cifar100}, although SD converges a little bit slower than the original baseline (solid blue line) at the first stage, they finally obtain similar stability. This is due to the introduction of the additional learning process, semantic clustering. \textbf{(c)} Although we design two semantic prior schemes, SD\_v1 and SD\_v2, they both show very consistent convergence, where the red and black solid lines even overlap with each other in Figure~\ref{fig:cifar100}. \textbf{(d)} The fluctuation in ResNext is due to the non-averaged loss value in the original code for each epoch. From the above observation, it is evident that the introduction of semantic clustering achieves stable and fast convergence, theoretically qualified to yield an optimal classification mapping. 

\setlength{\tabcolsep}{4pt}
\begin{table}
\begin{center}
\begin{tabular}{c|c|c|c|c|c|c}
\hline
Method & Solver & $\alpha_{1}$ & $\alpha_{2}$ & OT & RT & SD(ours)\\
\hline
MLP-3\cite{nokland2019training} & adam & $1$ & $0.5$ & 91.5 & 91.77 & \textbf{91.78}\\
VGG8b~\cite{nokland2019training} & adam & $1$ & $0.3$ & 95.45 & 95.47 & \textbf{95.53}\\
VGG8b(multi=2.0)~\cite{nokland2019training} & adam & $1$ & $0.3$ & 95.33 & 95.52 & \textbf{95.54}\\
\hline
\end{tabular}
\end{center}
\vspace{-2mm}
\caption{Classification accuracy on FashionMNIST.}
\label{table:mnist}
\end{table}
\setlength{\tabcolsep}{1.4pt}

\setlength{\tabcolsep}{4pt}
\begin{table}
\begin{center}
\begin{tabular}{c|c|c|c|c|c|c}
\hline
Method & Solver & $\alpha_{1}$ & $\alpha_{2}$ & OT & RT & SD(ours)\\
\hline
VGG8b~\cite{nokland2019training} & adam & $1$ & $0.5$ & 94.12 & 94.14 & \textbf{94.32}\\
ResNet18\cite{he2016deep} & adam & $1$ & $0.5$ & 93.45 & 93.57 & \textbf{93.62}\\
DenseNet-40-12~\cite{huang2017densely} & sgd & $1$ & $0.5$ & 94.68 & 94.79 & \textbf{94.92}\\
Wresnet-28-10~\cite{zagoruyko2016wide} & sgd & $1$ & $0.5$ & 94.52 & 94.58 & \textbf{94.80}\\
ResNext~\cite{xie2017aggregated} & sgd & $1$ & $0.5$ & 96.16 & 96.26 & \textbf{96.30}\\
\hline
\end{tabular}
\end{center}
\vspace{-2mm}
\caption{Classification accuracy on CIFAR10.}
\label{table:cifar10}
\end{table}
\setlength{\tabcolsep}{1.4pt}

\paragraph{Classification Accuracy}
We summarize the classification accuracy in Table~\ref{table:mnist},~\ref{table:cifar10}, and~\ref{table:cifar100}. \textbf{(a)} Generally, SD obtains almost the highest classification accuracy across the three benchmarks for all the compared classifiers. These classifiers include two mainstream solvers, adam~\cite{kingma2014adam} and sgd~\cite{bottou2010large}, but SD leads the performances in both situations. \textbf{(b)} Although the improvement brought by SD is limited for Fashion MNIST, this is mainly due to the relatively simple classification task and the limited number of classes. When it comes to CIFAR100 as shown in Table~\ref{table:cifar100}, SD always yields 1-3\% increase in accuracy compared with OT. \textbf{(c)} We can observe that RT in some special situations achieves high performance, such as RT winning SD in the case of Wresnet-28-10. However, its performance is not as stable as SD, which even yields lower classification accuracy than OT, such as that in the case of ResNet101. These observations imply that the proposed smooth semantic clustering algorithm can effectively enhance the performance of state-of-the-art classifiers, preserving a very stable learning state at the same time, potentially leading to its broader applicability. 

Compared with the recent publication~\cite{roy2020tree}, which proposes a network learning algorithm that organizes the incrementally learning data into feature-driven super-class and improves upon existing hierarchical CNN models by introducing the capability of self-growth, so that the finer classification is done. This idea, to a certain degree, shares a similar concept with our idea, except that we do not need to label data with super-class and keep the same hierarchical structure during the learning process. We compare its results with ours in Table~\ref{table:tree-cifar10} and Table~\ref{table:tree-cifar100}, respectively. It can be seen from them that, although the Tree-CNN models provide a competitive accuracy with its base network VGG-11, it shows no advantages over our SD models. SD models obtain a more than 4\% advantage over incremental learning methods (VGG11 and Tree-CNN in Table~\ref{table:tree-cifar10}) on CIFAR 10 and averagely 5\% higher than incremental learning methods on CIFAR100 considering the test classification accuracy. It demonstrates that our proposed high-level semantic clustering algorithm, in a direct supervised learning, could further improve the adaptive ability towards data, and keep a stable learning process, which is further verified in the following sections. Most importantly, we focus on the exploration towards the self-deducing ability of CNN models, which is different from all the above-mentioned ideas. 

\setlength{\tabcolsep}{4pt}
\begin{table}
\begin{center}
\begin{tabular}{c|c|c|c|c|c|c|c}
\hline
Method & Solver & $\alpha_{1}$ & $\alpha_{2}$ & OT & RT & SD\_v1 & SD\_v2\\
\hline
VGG8b~\cite{nokland2019training} & adam & $1$ & $0.5$ & 73.85 & 74.78 & \textbf{74.95} & 74.83\\
ResNet50~\cite{he2016deep} & sgd & $1$ & $0.5$ & 73.78 & 76.36 & 75.59 & \textbf{76.64}\\
DenseNet-40-12~\cite{huang2017densely} & sgd & $1$ & $0.5$ & 74.89 & 75.82 & \textbf{76.26} & 75.73\\
Wresnet-28-10~\cite{zagoruyko2016wide} & sgd & $1$ & $0.5$ & 76.98 & \textbf{77.62} & 77.54 & 77.59\\
ResNet101~\cite{he2016deep} & sgd & $1$ & $0.5$ & 75.3 & 74.45 & 75.51 & \textbf{76.29}\\
ResNet152~\cite{he2016deep} & sgd & $1$ & $0.3$ & 72.21 & 73.25 & 74.38 & \textbf{74.40}\\
\hline
\end{tabular}
\end{center}
\vspace{-2mm}
\caption{Classification accuracy on CIFAR100.}
\label{table:cifar100}
\end{table}
\setlength{\tabcolsep}{1.4pt}

\setlength{\tabcolsep}{1.5pt}
\begin{table}
\begin{center}
\begin{tabular}{c|c|c|c|c|c|c}
\hline
Model & VGG11 & Tree-CNN & VGG8b & ResNet18-SD & DenseNet-SD & WresNet-SD\\
\hline
Test Accuracy & 90.51 & 86.24 & 94.32 & 93.62 & 94.92 & 94.80\\
\hline
\end{tabular}
\end{center}
\vspace{-2mm}
\caption{Comparison with Tree-CNN on Cifar10, where SD refers to models that are applied with our proposed algorithm. VGG11 and Tree-CNN are trained by "old" and "new" data in an incremental way~\cite{roy2020tree}.}
\label{table:tree-cifar10}
\end{table}
\setlength{\tabcolsep}{1.5pt}

\setlength{\tabcolsep}{1.5pt}
\begin{table}
\begin{center}
\begin{tabular}{c|c|c|c|c|c|c}
\hline
Model & VGG11 & Tree-CNN5 & Tree-CNN10 & Tree-CNN20 & VGG8b-SD & Wresnet-28-10-SD\\
\hline
Test-Acc & 72.23 & 69.85 & 69.53 & 68.49 & 74.95 & 77.54\\
\hline
\end{tabular}
\end{center}
\vspace{-2mm}
\caption{Comparison with Tree-CNN on Cifar100, where Test-Acc stands for the Test Accuracy. SD refers to the corresponding models that are applied with our proposed algorithm. VGG11 and Tree-CNN are trained by "old" and "new" data in an incremental way~\cite{roy2020tree}.}
\label{table:tree-cifar100}
\end{table}
\setlength{\tabcolsep}{1.5pt}

\subsection{Results in Noisy Data Environment}
In this section, we evaluate the proposed algorithm in noisy data environments. We produce a noisy data environment by adding noise labels to the original data sources. Specifically, we implement this operation on CIFAR10 and CIFAR100, where 10\% of the training data in each data set are randomly labeled by incorrect labels that belong to the same colony with the correct labels. For example, if the image is labeled correctly by ``cat", then we randomly search another class label in the ``animal" cluster such as ``dog" as the replacement of the label ``cat".  
\paragraph{Convergence Performance} The comparative results are shown in Figures \ref{fig:cifar10} and \ref{fig:cifar100}, from which we can see that \textbf{(a)} SD maintains the same learning stability as that in original environment. It even surpasses the baseline OT by convergence speed in some cases, such as DenseNet-40-12 and Wresnet-28-10 on CIFAR10. \textbf{(b)} SD generally converges faster than RT, especially in the case of Wresnet-28-10. It shows SD works better assisting the classifier to execute reasonable classification decisions in noisy situations, which exhibits good robustness of the proposed algorithm. \textbf{(c)} SD with the composite loss function ``$\mathcal{L}$" shows perfect robustness across both shallow and deep networks. 
Thus, SD is expected to identify the optimal classification theoretically. 

\setlength{\tabcolsep}{4pt}
\begin{table}
\begin{center}
\begin{tabular}{c|c|c|c|c|c|c}
\hline
Method & Solver & $\alpha_{1}$ & $\alpha_{2}$ & OT & RT & SD(ours)\\
\hline
VGG8b~\cite{nokland2019training} & adam & $1$ & $0.3$ & 89.71 & \textbf{90.52} & 90.33\\
ResNet18~\cite{he2016deep} & adam & $1$ & $0.3$ & 89.22 & \textbf{90.71} & 90.32\\
DenseNet-40-12~\cite{huang2017densely} & sgd & $1$ & $0.5$ & 91.47 & 92.16 & \textbf{92.25}\\
wresnet-28-10\cite{zagoruyko2016wide} & sgd & $1$ & $0.5$ & 89.07 & 87.67 & \textbf{89.21}\\
ResNext\cite{xie2017aggregated} & sgd & $1$ & $0.3$ & 91.29 & 92.17 & \textbf{92.53}\\
\hline
\end{tabular}
\end{center}
\vspace{-2mm}
\caption{Classification on CIFAR10 with noisy labels.}
\label{table:cifar10n}
\end{table}
\setlength{\tabcolsep}{1.4pt}

\setlength{\tabcolsep}{4pt}
\begin{table}
\begin{center}
\begin{tabular}{c|c|c|c|c|c|c|c}
\hline
Method & Solver & $\alpha_{1}$ & $\alpha_{2}$ & OT & RT & SD\_v1 & SD\_v2\\
\hline
VGG8b~\cite{nokland2019training} & adam & $1$ & $0.5$ & 67.68 & 68.72 & 68.89 & \textbf{68.95}\\
ResNext~\cite{he2016deep} & sgd & $1$ & $0.5$ & 75.48 & 74.51 & 75.03 & \textbf{75.65}\\
DenseNet-40-12~\cite{huang2017densely} & sgd & $1$ & $0.5$ & 70.25 & \textbf{72.80} & 72.61 & 72.09\\
wresnet-28-10~\cite{zagoruyko2016wide} & sgd & $1$ & $0.5$ & 71.42 & 71.79 & 71.60 & \textbf{72.59}\\
ResNet101\cite{he2016deep} & sgd & $1$ & $0.5$ & 68.93 & 67.97 & 68.71 & \textbf{69.75}\\
\hline
\end{tabular}
\end{center}
\vspace{-2mm}
\caption{Classification on CIFAR100 with noisy labels.}
\label{table:cifar100n}
\end{table}
\setlength{\tabcolsep}{1.4pt}

\paragraph{Classification Accuracy} The comparative results are shown in Tables~\ref{table:cifar10n} and~\ref{table:cifar100n}. We can observe that 
\textbf{(a)} SD, in general, surpasses OT by 1-2\%. \textbf{(b)} Although RT surpasses SD in some cases, their results are very close. SD is always consistent for all the compared models. \textbf{(c)} RT is less robust than SD for its poor performance in some cases with much lower accuracy than OT, such as Wresnet-28-10 on CIFAR10, and ResNext and ResNet101 on CIFAR100. 

These observations indicate that the proposed deduction learning by semantic clustering not only enhances the classification performance but also improves the generalization for a given classifier. 
From the above experiments, it is evident that the proposed semantic clustering method can help the model achieve more accurate classification decisions. Although the semantic prior-based opposite label search provides rough information, it can aid the model to deduce high-level semantic expression along with the entire learning process, realizing the experience accumulation and basic cognitive learning. Thus, it could be an excellent plug-in module that could be applied in other supervised learning, few-shot learning, zero-shot learning, or even semi-supervised learning where each learning stage could be a better fit, generalized, and becoming much more robust. In the meanwhile, from the perspective of calculation, the proposed mechanism of deduction learning by the opposite semantic constraint only introduces one more loss item, which is only the tenth level of the order of magnitudes. Compared with matrix multiplication of any two layers during the training process which has the million level of the order of magnitudes, our proposed model is capable of keeping the time complexity of calculation, while its superior stability and robustness make it easy to be generalized to other computer vision tasks.

\section{Conclusion}
In this paper, we have proposed a deduction learning approach to boost the gain of high-level semantic clustering. We have demonstrated that if a classifier can perform further independent mapping in the semantic space, it will help the model achieve higher classification performance with better generalization ability and robustness. The proposed smooth semantic clustering algorithm ensures label learning and semantic deduction being processed in the same timeline so as to form a basic cognition. Extensive experiments across various classifiers on different datasets demonstrate the superiority of the proposed method toward further enhancing state-of-the-art classification performance.

\section*{Acknowledgement}
The work was supported in part by The National Aeronautics and Space Administration (NASA) under the grant no. 80NSSC20M0160, and the Nvidia GPU grant.





\bibliography{main}
\end{document}